\documentclass{amsart}
\usepackage{amsmath,amssymb,bm}
\usepackage{amssymb,bm}
\usepackage{graphicx}
\usepackage{color}

\newtheorem{theorem}{Theorem}

\newtheorem{prop}[theorem]{Proposition}

\newtheorem{remark}{Remark}

\DeclareMathOperator*{\argmin}{arg\,min}

\newtheorem{lem}[section]{Lemma}

\begin{document}

\author{Chikara Nakayama} 
\address{Graduate School of Economics, Hitotsubashi University, 2-1 Naka, Kunitachi, Tokyo 186-8601, Japan} 
\email{c.nakayama@r.hit-u.ac.jp}

\author{Tsuyoshi Yoneda} 
\address{Graduate School of Economics, Hitotsubashi University, 2-1 Naka, Kunitachi, Tokyo 186-8601, Japan} 
\email{t.yoneda@r.hit-u.ac.jp}




\title[approximation of dynamical systems by  RNNs]
{Explicit construction of recurrent neural networks effectively approximating discrete dynamical systems
} 

\begin{abstract} 
We consider arbitrary bounded discrete time series originating from dynamical system with recursivity.
More precisely, we provide an explicit construction of recurrent neural networks which 
effectively approximate the corresponding discrete dynamical systems.
\end{abstract}

\maketitle

\section{Introduction}

In this paper we provide an explicit construction of recurrent neural networks (RNNs) which effectively approximate  discrete dynamical systems.
In general, Lyapinov indexes for reconstructed and original systems have a gap (see Dechert-Gen\c{c}ay \cite{DG} and Berry-Das \cite{BD}), and our explicit construction of RNNs at least minimize the maximal Lyapunov exponent.
As far as the authors are aware, none of the numerous works to date attempted to
explicitly construct RNNs in terms of an elementary algebraic approach, and in this paper we initiate it (for investigating the universality of periodic points based on the similar algebraic approach, see \cite{NY0}). 

First we define the dynamical system in a delay coordinate.
For the dimension of the dynamical system $L\in\mathbb{Z}_{\geq 1}$,  
let $\Phi:[-1,1]^L\to[-1,1]^L$ ($(w_1,w_2,\cdots,w_L)^T\mapsto (x_1,x_2,\cdots,x_L)^T$) be such that
\begin{equation}\label{delay coordinate}
w_1=x_2,\quad w_2=x_3,\cdots, w_{L-1}=x_L
\end{equation}
 (i.e. discrete dynamical system in a delay coordinate, see \cite{T}, and see also \cite{NS}).
Let us assume that a time series $y:\mathbb{Z}\to[-1,1]$ satisfies
\begin{equation*}
    Y_{t+1}=\Phi(Y_t)\quad\text{for}\quad t\in\mathbb{Z},
\end{equation*}
where
\begin{equation*}
 Y_t
:=(y(t),y(t-1),\cdots,y(t-L+1))^T.
\end{equation*}
From this $\Phi$, we define $\lambda\in\mathbb{R}$ to be the maximum Lyapunov exponent, and assume it is finite, as follows: 
\begin{equation*}
e^\lambda:=\sup_{\stackrel{W,W'\in[-1,1]^L,}{W\not=W'}}\frac{|\Phi(W)-\Phi(W')|}{|W-W'|}<\infty.
\end{equation*}
Also we assume recursivity to  $y$:
For any $\epsilon>0$, 
there is $\{t_j\}_{j=1}^\infty\subset\mathbb{Z}$
($t_j\to-\infty$ for $j\to\infty$) such that 
\begin{equation*}
    |Y_0-Y_{t_j}|<\epsilon.
\end{equation*}
Now we define the conventional RNNs.
For initial hidden state vector $r(0)\in\mathbb{R}^N$, initial data $\hat y(0)\in[-1,1]$,
recurrent weight matrix $W\in\mathbb{R}^{N\times N}$,
input weight matrix (vector) $W^{in}\in\mathbb{R}^{N\times 1}$, output weight matrix (vector) $W^{out}\in\mathbb{R}^{1\times N}$ 
 and activate function $h:\mathbb{R} \to [-1,1]$, 
the RNNs can be expressed as follows:
\begin{equation}\label{RC}
\begin{split}
&
\begin{cases}
 r(t+1)&:= h(W r(t)+W^{in}\hat y(t)),\\
 \hat y(t+1)&:=W^{out}r(t+1)
 \quad\text{for}\quad t\geq 0.
\end{cases}
\end{split}
\end{equation}
The main theorem is as follows:

\begin{theorem}\label{L fixed case}
Let $y:\mathbb{Z}\to[-1,1]$ be a time series generated by the discrete dynamical system in a delay coordinate $\Phi:[-1,1]^L\to[-1,1]^L$ ($Y_{t-1}\mapsto Y_t$) 
with recursivity and  a finite $\lambda\geq 0$.
Then for any large $K \in \mathbb{Z}_{\ge1}$ and for any constant $C$ which is slightly larger than $1$, there exist  $N\le K^L$, $h$, explicit $W$, $W^{in}$ and $W^{out}$ which are 
composed of $y(t)$ ($t\in\mathbb{Z}_{\leq 0}$),
an initial hidden state vector $r(0)$, and initial data $\hat y(0)$ 
such that the following holds:
\begin{equation}
\label{ineq}
\displaystyle|\hat y(t)-y(t)|\leq (2t+1)e^{\lambda t}\frac{\sqrt LC}{K}\quad \text{for}\quad t\geq 0.
\end{equation}
This is only on the order of log worse (since $t=e^{\log t}$). Thus, this RNNs construction seems effective for re-expressing the original system.
\end{theorem}

\section{Proof of main theorem}

First let us discretize the range $[-1,1]$ as follows:
For $K\in\mathbb{Z}_{\geq 1}$,  we choose $\{a_k^K\}_{k=1}^{K}\subset[-1,1]$ such that 
\begin{itemize}

\item
$-1
<a_{1}^K<a_{2}^K<\cdots<a_K^K < 1$,

\item

$a_i^K \not=0$ for any $1 \le i \le K$ and \\
$\displaystyle\sum_{\ell=1}^L \frac{a_{k'_{\ell}}^K}{a_{k_{\ell}}^K}\not=L$ for each $(k_1,\ldots,k_L)\not=(k'_1,\ldots,k'_L)$,

\item

$\sup\left\{\frac{a_{k+1}^K-a_{k-1}^K}{2}\ (1<k<K), \frac{a_{1}^K+a_{2}^K}{2}+1, 
1-\frac{a_{K-1}^K+a_{K}^K}{2}\right\}\leq C/K$

\item 
For any $t$ with $-L+1 \le t \le 0$ and for any $k$ with $1 \le k < K$, 
\begin{equation}\label{adjustemt of initial data}
    y(t) \not=\frac{a_k^K+a_{k+1}^K}2. 
\end{equation}
\end{itemize}
We need the last condition for appropriately separating 
the time series $y$ into training and inference phases.
Note that the above second condition is almost always satisfied in the following sense.
\begin{prop}\label{discrete-range}
There is a closed set $Z$ of $[-1,1]^K$ whose Lebesgue measure is zero such that the above second condition is satisfied as soon as $(a^K_1,\ldots, a^K_K)$ is not in $Z$.
\end{prop}

\begin{proof}
  We drop the superscipt $K$ for simplicity. 
Regard $a_i$ ($1 \le i \le K$) as indeterminates. 
 Clearly the measure of the set where $a_i=0$ is zero. 
  Under the assumption that any $a_i$ is not zero, the condition $\sum_{\ell=1}^L \frac{a_{k'_{\ell}}}{a_{k_{\ell}}}\not=L$ is equivalent to 
$\sum_{\ell} a_{k_1}\cdots a_{k'_{\ell}} \cdots a_{k_L}-L \prod_{\ell} a_{k_{\ell}}\not=0$, the left-hand-side of which is a nonzero polynomial whenever $k_{\ell} \not=k'_{\ell}$ for some $\ell$ (then the term $\prod_{\ell} a_{k_{\ell}}$ does not vanish).
  Hence by the following simple fact, the non-equality holds for almost all values of $a_i$:
\begin{center}
$\bullet$  Let $f$ be a nonzero polynomial in $M$ variables.   Then
the Lebesgue measure of the zero point set $\{x \in \Bbb R^M : f(x)
=0\}$ is zero.
\end{center}
\end{proof}
We now discretize $y$ (in the training phase) as follows:
\begin{equation}\label{discrete}
\bar y_K(t):=\argmin_{a\in\{a^K_k\}_{k=1}^{K}}|y(t)-(a-0)|\quad\text{for}\quad t\in\mathbb{Z}_{\leq 0},
\end{equation}
where $a-0:=a-\varepsilon$ for any sufficiently small $\varepsilon>0$.
From this discretization, we define  ``dictionary", in other word, ``key-value pairs". 
Let $\sigma_n$ ($n=1,2,\cdots,N$) be a permutation operator, namely, a map 
\begin{equation*}
\sigma_n^K:\{1,2,\cdots, L\}\to\{a_{1}^K,a_{2}^K,\cdots,a_{K-1}^K,a_{K}^K\}
\end{equation*}
($\ell\mapsto \sigma_n^K(\ell)$) with $\sigma_n^K\not=\sigma_{n'}^K$ ($n\not=n'$), 
and we impose the following two conditions for determining $N$:
\begin{equation}\label{all patterns}
\begin{cases}
\text{For any $t\in\mathbb{Z}_{\leq  0}$, there is $n\in\{1,\cdots,N\}$ such that}\\
\qquad
\text{$\sigma_n^K(\ell)=\bar y_K(t-\ell)$ for $\ell=1,2,\cdots, L$},\\
\text{For any  $n\in\{1,\cdots,N\}$ there is $t\in\mathbb{Z}_{\leq  0}$ such that}\\
\qquad\text{$\sigma_n^K(\ell)=\bar y_K(t-\ell)$ for $\ell=1,2,\cdots, L$.} 
\end{cases}
\end{equation}

Note that  $N\leq K^L$ due to the sequence with repetition.
We now define the dictionary $\{(\sigma_n,a_{k(n)})\}_{n=1}^N$, in other word, key-value pairs (i.e. $\sigma_n$ is the ``key" and $a_{k(n)}$ is the corresponding ``value").
By \eqref{all patterns},
for any key $\sigma_n^K$, there exists 
 a $t'\in\mathbb{Z}_{\leq 0}$ and then we choose such a $t'$ and the corresponding $k(n)\in\{1,2,\cdots, K\}$ such that 
\begin{equation}\label{next value}
 \text{key:}\quad \sigma_n^K(\ell)=\bar y_K(t'-\ell)\quad (\ell=1,2,\cdots, L)\quad\text{and}\quad \text{value:}\quad a_{k(n)}^K:=\bar y_K(t').
\end{equation}
  Now, by using these key-value pairs, we first construct a time series $y^*(t)$ $(t \ge0)$, and then show an estimation analogous to (\ref{ineq}), that is, the inequality obtained by replacing $\hat y$ in (\ref{ineq}) by $y^*$. 
  After that, we give $h, W, W^{in}, W^{out}, r(0)$ and $\hat y(0)$ 
such that $\hat y=y^*$, which completes the proof. 

  We define $y^*(t)$ by induction. 
  Let $y^*(t)=\bar y_K(t)$ for $t=0,-1,\ldots,-L$. 
  Then the permutation operator $\ell \mapsto y^*(1-\ell)$ is in the dictionary by recursivity and (\ref{adjustemt of initial data}).
  Let $t_0>0$. 
  Assume that we have defined $y^*(t)$ for $-L\le t<t_0$ and assume that 
we can find the permutation operator $\ell \mapsto y^*(t_0-\ell)$ is in the dictionary, that is, it is $\sigma_n$ for some $n$. 
  Define $y^*(t_0):=a_{k(n)}$. 
  Then we see that the permutation operator $\ell \mapsto y^*(t_0-\ell+1)$ also in the dictionary as follows.
  Recall that for some $t'\leq 0$, $\sigma_n(\ell)=\bar y_K(t'-\ell)$  ($\ell=1,2,\cdots, L$) 
and $a_{k(n)}=\bar y_K(t')$.
  Hence the permutation operator concerned coincides with $\ell \mapsto \bar y_K(t'+1-\ell)$.  
  If $t'<0$, there is $\tilde n$ such that 
$\bar y_K(t'-\ell+1)=\sigma_{\tilde n}(\ell)$ ($\ell=1,2,\cdots, L$) by \eqref{all patterns}.
  If $t'=0$, we just apply recursivity and (\ref{adjustemt of initial data}).
  Thus the induction goes and we have defined $y^*$. 

  We estimate the difference of $y^*(t)$ and $y(t)$, which eventually implies 
\eqref{ineq}.
  The case $t=0$ is by the third condition in the beginning of this section. 
Let 
\begin{equation*}
\begin{split}
\overline Y_t
&:=(\bar y_K(t),\bar y_K(t-1),\cdots,\bar y_K(t-L+1))^T \quad (t\le0),\\
 Y^*_t
&:=(y^*(t),y^*(t-1),\cdots,y^*(t-L+1))^T \quad (t \ge 0). 
\end{split}
\end{equation*}
Then we see that, for any $t\in\mathbb{Z}_{\geq 1}$, there exists $t'\in\mathbb{Z}_{\leq 0}$ such that
\begin{equation}\label{uniform estimates}
\begin{cases}
\displaystyle
\overline Y_{t'}= Y^*_{t},\quad |Y_{t'}-Y^*_{t}|\leq \frac{\sqrt LC}{K},\\
Y_{t'}=\Phi(Y_{t'-1}),\\
\displaystyle
\overline Y_{t'-1}=Y^*_{t-1},\quad
|Y_{t'-1}- Y^*_{t-1}|\leq \frac{\sqrt LC}{K}.
\end{cases}
\end{equation}
Recall $e^{\lambda} \ge1$, and then we have 
\begin{equation*}
\begin{split}
|y^*(t)-y(t)|\leq
& |Y^*_{t}-Y_{t}|\\
\leq 
&
 |Y_{t'}-Y_{t}|+\frac{\sqrt LC}{K}\\
\leq
&
e^\lambda |Y_{t'-1}-Y_{t-1}|+\frac{\sqrt LC}{K}\\
\leq 
&e^\lambda |Y^*_{t-1}-Y_{t-1}|+e^\lambda\frac{\sqrt LC}{K}
+
\frac{\sqrt LC}{K}\\
\leq 
&e^\lambda |Y^*_{t-1}-Y_{t-1}|+2e^\lambda\frac{\sqrt LC}{K}
\leq\cdots\\
\leq
&e^{\lambda t}|Y^*_{0}-Y_{0}|+2te^{\lambda t}\frac{\sqrt LC}{K}\\
\leq
&(2t+1)e^{\lambda t}\frac{\sqrt LC}{K}.
\end{split}
\end{equation*}
  Next we explicitly construct effective
recurrent neural networks.
Let $\sigma^*_n$ ($n=1,2,\cdots, N$) be an adjoint type of permutation operator, more precisely, let
\begin{equation}\label{adjoint sigma}
\displaystyle\sigma^*_n(\ell-1):=\frac{1}{\sigma_n(\ell)}
\end{equation}
for $\ell\in\{1,\cdots, L\}$ and $n\in\{1,\cdots, N\}$. 
The definition of the following $N\times N$ matrix $X$ 
is the crucial to prove the main theorem:
\begin{equation*}
X:=\left[h\left(\sum_{\ell=1}^L\sigma^*_i(\ell-1)\sigma_j(\ell)\right)\right]_{i,j}
\end{equation*}
for some  $h:\mathbb{R}\to[-1,1]$.
Also let $G$ be such that
\begin{equation*}
G:=\left\{\sum_{\ell=1}^L\sigma^*_i(\ell-1)\sigma_j(\ell): i,j\in\{1,2,\cdots,N\}\right\}.
\end{equation*}
Note that $L\in G$.

\begin{lem}\label{l:regular}\label{key lemma}
$X$ is a regular matrix for almost all $h\colon \Bbb R \to [-1,1]$ in
the following sense.
Let $\gamma_1,\ldots,\gamma_M$ be the all elements of $G$.
Then there is a closed set $Z$ of $\Bbb R^M$ whose Lebesgue measure is
zero such that
$X$ is regular as soon as $(h(\gamma_1),\ldots,h(\gamma_M))$ is not in $Z$.
\end{lem}

\begin{proof}  Take a bijection $x\colon \gamma_m \mapsto x_m$ from $G$
to the set of indeterminates $\{x_1,\ldots, x_M\}$.
Consider the determinant $D$ of the matrix $[x(\sum
\sigma_i^*(\ell-1)\sigma_j(\ell))]_{i,j\in N}$ with polynomial
coefficients.
Then $D$ is a nonzero polynomial because it is a monic polynomial of
degree $N$ with respect to the indeterminate $x(L)$ (the leading term
$x(L)^N$ comes from the diagonal).
Now we apply the simple fact in the proof of Proposition \ref{discrete-range} to $D$.
Let $Z$ be the zero point set of $D$.  Assume that
$(h(\gamma_1),\ldots,h(\gamma_M))$ is not in $Z$.
  Then $\det X$ defined by $h$ is not zero so that $X$ is regular.
\end{proof}
With the aid of the inverse of $X$, we can provide 
the $N\times N$ recurrent weight matrix $W$ as follows:
\begin{equation}\label{key-formula}
W:=YX^{-1}\quad\text{with}\quad Y:
=
\left[\sum_{\ell=1}^{L-1}\sigma_i^*(\ell)\sigma_j(\ell)\right]_{i,j}.
\end{equation}

\begin{remark}
The rank of $W$ is less than $L-1$, because Y is the product of the $N\times(L-1)$-matrix: $[\sigma_i^*(\ell-1)]_{i,\ell}$ and $(L-1)\times N$-matrix: 
$[\sigma_j(\ell)]_{\ell,j}$. 
\end{remark}
  We now define the input weight vector $W^{in}$ and the the output weight matrix (vector) $W^{out}$ as follows:
\begin{equation*}
\begin{split}
W^{in}&:=\sigma^*(0),\\
W^{out}&:=(a_{k(1)},a_{k(2)},\cdots,a_{k(N)})X^{-1},
\end{split}
\end{equation*}
where $\sigma^*(\ell):=(\sigma^*_1(\ell),\sigma^*_2(\ell),\cdots,\sigma_N^*(\ell))^T$ 
 for $0 \le \ell \le L-1$. 
  Let us set the initial hidden state vector $r(0)$ and the initial data $\hat y(0)$ as follows:
\begin{equation*}
\begin{split}
r(0)&=\left[h\left(\sum_{\ell=1}^L\sigma^*_i(\ell-1)\bar y_K(-\ell)\right)\right]_{i},\\ 
 \hat y(0)&:=\bar y_K(0),
\end{split}
\end{equation*}
 then we have the desired RNNs. 
  Let $\hat y$ be the time series generated by these initial data. 
  Now we certify $\hat y=y^*$, which completes the proof of Theorem.
For $t \ge 0$, 
  the permutation operator $\ell \mapsto y^*(t-\ell)$ is $\sigma_n$ for some $n$. 
  Let $n_t$ be this $n$. 
Then we define the column vector
\begin{equation*}
e_{n_t}:=(\underbrace{\overbrace{0,0,\cdots,0,1}^{n_t},0\cdots,0}_{N})^T.
\end{equation*}
  We prove $\hat y(t)=y^*(t)$ for $t \ge0$ together with $r(t)=Xe_{n_t}$ by induction. 
  The case $t=0$ is by the definition. 
  Assume that $t\ge0$ and they hold for $0,1,\ldots, t$. 
Then  we have 
\begin{equation*}\label{error cluster}
\text{key:}\quad y^*(t-\ell)=\sigma_{n_t}(\ell) \quad\text{for}\quad \ell=1,2,\cdots, L.
\end{equation*}
Thus we have 
\begin{equation*}
Wr(t) 
=WXe_{n_t}=Ye_{n_t}
=\sum_{\ell=1}^{L-1}\sigma^*(\ell) y^*(t-\ell),
\end{equation*}
\begin{equation*}\label{determined value}
    r(t+1)=h\left(Ye_{n_t}+\sigma^*(0) y^*(t)\right)=Xe_{n_{t+1}} {\text { and} }\\
\end{equation*}

\begin{equation*}\hat y(t+1)=W^{out}r(t+1)=(a_{k(1)},\ldots,a_{k(N)})e_{n_{t+1}}=a_{k(n_{t+1})}=y^*(t+1).
\end{equation*}
  Hence they hold also for $t+1$. 



\vspace{0.5cm}
\noindent
{\bf Acknowledgments.}\  
The first author thanks J.\ C.\ for leading him to join this work.
Research of  TY  was partly supported by the JSPS Grants-in-Aid for Scientific
Research 24H00186.
Research of  CN was partly supported by the JSPS Grants-in-Aid for Scientific
Research 21K03199.




\begin{thebibliography}{99}



\bibitem{BD}
T. Berry and S. Das,
\emph{Learning theory for dynamical systems},
SIAM J. Appl. Dyn. Syst., {\bf 22},  (2023) 10.1137/22M1516865.

%
%
%
%
%

\bibitem{DG}
W. Dechert and R. Gen\c{c}ay,
\emph{The topological invariance of lyapunov exponents in embedded dynamics},
Phys. D, {\bf 90}  (1996) 40–55.



%
%




%
%
%
%















\bibitem{NS}
 K. Nakai  and Y. Saiki,
\emph{Machine-learning construction of a model for a macroscopic fluid variable using the delay-coordinate of a scalar observable}, 
Discr. Conti. Dyn. Sys.-S, {\bf 14}, (2021) 1079-1092.


\bibitem{NY0}
C. Nakayama and T. Yoneda, 
\emph{Universality of almost periodicity in bounded discrete time series},
arXiv:2310.00290.


\bibitem{T}
F. Takens, 
\emph{Detecting strange attractors in turbulence}, Dynamical Systems of
Turbulence, Lecture Notes in Mathematics, {\bf 898}, 
Springer-Verlag, Berlin, (1981) 366-381.



\end{thebibliography}
\end{document}